\documentclass[twocolumn,10pt,journal,twoside]{IEEEtran}
\usepackage[cmex10]{amsmath}
\usepackage{amssymb}
\usepackage{amsfonts}
\usepackage{amsthm}
\usepackage{algorithm}
\usepackage{algpseudocode}
\usepackage{graphicx}
\usepackage{epsfig}
\usepackage{cite}
\usepackage{tensor}
\usepackage[caption=false,font=footnotesize]{subfig}
\usepackage{color}
\usepackage{makecell}
\usepackage{float}
\usepackage{comment}
\usepackage{subfig}
\usepackage[dvipsnames]{xcolor}

\usepackage{hyperref}
\hypersetup{hypertex=true,
	colorlinks=true,
	linkcolor=blue,
	anchorcolor=blue,
	citecolor=blue}

\graphicspath{{figures/}}
\DeclareGraphicsExtensions{.eps}
\newtheorem{proposition}{Proposition}
\newtheorem{lemma}{Lemma}
\theoremstyle{definition}
\newtheorem{remark}{Remark}
\newtheorem{assumption}{Assumptions}
\newtheorem{definition}{Definition}

\newcommand\T{{\hspace{-0pt}\intercal}}

\DeclareMathOperator{\sgn}{sgn}

\DeclareMathOperator{\erf}{erf}

\begin{document}

\title{
Sliding Mode Control for 3-D Uncalibrated and Constrained Vision-based Shape Servoing within Input Saturation
}

\author{
Fangqing Chen $^*$\\
University of Toronto

\thanks{
Copyright may be transferred without notice, after which this version
may no longer be accessible.
}
\thanks{
$^*$ Corresponding Author.
}

}

\bstctlcite{IEEEexample:BSTcontrol}

\maketitle

\begin{abstract}
This paper designs a servo control system based on sliding mode control for the shape control of elastic objects.
In order to solve the effect of non-smooth and asymmetric control saturation, a Gaussian-based continuous differentiable asymmetric saturation function is used for this goal.
The proposed detection approach runs in a highly real-time manner.
Meanwhile, this paper uses sliding mode control to prove that the estimation stability of the deformation Jacobian matrix and the system stability of the controller are combined, which verifies the control stability of the closed-loop system including estimation.
Besides, an integral sliding mode function is designed to avoid the need for second-order derivatives of variables, which enhances the robustness of the system in actual situations.
Finally, the Lyapunov theory is used to prove the consistent final boundedness of all variables of the system.
\end{abstract}


\IEEEpeerreviewmaketitle

\section{Introduction}\label{section1}
The deformable object manipulation (DOM) receives considerable attention in the field of robotics.
Also, this application can be seen everywhere, such as:
industrial processing \cite{grutzner2007development},
medical surgery \cite{gerhardus2003robot}, 
furniture services \cite{soderlund2021robot}, and 
item package \cite{dubey2006packaging}. 
However, although a lot of research has been done on DOM, a complete manipulation framework system has not yet been formed \cite{ma2022active}. 
The biggest difficulty in this is the complex and unknown physical characteristics of the deformable linear object (DLO) is hard to obtain in the real application environment \cite{zou2022deep,zheng2016path}.
Currently, methods targeting DOM are broadly divided into learning-based and cybernetics-based.
This article designs the DOM from a control perspective.

To the best of our knowledge, this is the first attempt to design an SMC-based manipulation framework for DLO with the consideration of the non-symmetric saturation control issue. \cite{yu2018t,sheng2017adaptive,qi2021contour}, which helps to obtain more explainable contents in the physical applications \cite{sanchez2020blind}.

The key contributions of this paper are three-fold:
\begin{itemize}
\item 
\textbf{Construction of the Unified Shape Manipulation:}
This paper forms a unified manipulation framework from the control-based viewpoint.
The core modules are constructed in three parts, detection/extraction, approximation, and shape servoing control.
The proposed manipulation runs in a model-free manner and does not need any prior knowledge of the system model.

\item 
\textbf{Consideration of the Non-symmetric Saturation:}
This paper considers the common asymmetric and non-smooth control input saturation problems in practical applications and avoids the control input discontinuity problem caused by the traditional use of hard saturation measures by introducing a Gaussian saturation function.
\end{itemize}

\section{Related Work}
\subsection{Input Saturation}
In the real application environment, the input saturation issue that occurred in the control should be paid attention to improve the system performance \cite{tarbouriech2009anti,mohammad2008robust}.
About the solutions of nonlinear saturation, much research has been conducted in recent years, and is well addressed in \cite{qi2018trajectory,zhang2014neural}.
About the detailed survey of the control input saturation, we refer the readers to \cite{fallah2018observer,gao2021chebyshev}.

Input saturation of plants with uncertain models has been considered in developing anti-windup schemes in \cite{visioli2006anti,schiavo2022individualized}.
Some state-of-art methods are introduced in \cite{lee2015study}, and a unified framework incorporating the various existing anti-windup schemes is presented in \cite{schiavo2021optimized}. 
Model predictive control (MPC) plays the most important role in this question, output, or state constraint \cite{rodrigues2005robust}. 
However, the limitation of MPC is also obvious, it needs many online iterations, so it is not good for the specified tasks that need high real-time performance \cite{galuppini2018model}.
MPC must be completed between the two sampling instances, which increases the calculation burden for real-time control \cite{shi2014sampled}.
Furthermore, it should be noted that one critical assumption made in the most of the above research is that the actuator saturation is symmetric \cite{dubey2006packaging,hu2000null}.
The symmetry of the saturation function to a large degree simplifies the analysis of the closed-loop system \cite{yuan2014switching}. 

In this brief, we will investigate a sliding model control-based manipulation framework with also considers the nonlinear asymmetric saturation issue.

\section{PROBLEM FORMULATION}\label{section2}
\emph{Notation:}
In this paper, we use the following frequently-used notation:
Bold small letters, e.g., $\mathbf{m}$ denote column vectors, while bold capital letters, e.g., $\mathbf{M}$ denote matrices.

\subsection{Robot-Manipulation Model}
Consider the kinematic-controlled 6-DOF robot manipulator (i.e., the underlying controller can precisely execute the given speed command), $\mathbf{q} \in \mathbb{R}^{6}$ and $\mathbf{r} = \mathbf{f}_r(\mathbf{q}) \in \mathbb{R}^6$ are denoted by the robot's joint angles and end-effector's pose, respectively.
$\mathbf{f}_r(\mathbf{q})$ is the forward kinematics of the robot. 
Therefore, the standard velocity Jacobian matrix $\mathbf{J}_r$ of the manipulator can be obtained as follows: $\dot{\mathbf{r}} = \mathbf{J}_r(\mathbf{q}) \dot{\mathbf{q}}$ where $\mathbf{J}_r(\mathbf{q}) \in \mathbb{R}^{6 \times 6}$ is assumed to be exactly known.
In addition, we assume that the robot does not collide with the environment or itself during the manipulation process, i.e., collision avoidance is not the scope of the article \cite{qi2022model}.

\subsection{Visual-Deformation Model}
In this article, the centerline configuration of the object is defined as follows:
\begin{align}
\label{eq31}
\bar{\mathbf{c}}=
[\mathbf{c}_1^\T,\ldots,\mathbf{c}_N^\T]^\T \in \mathbb{R}^{2N}, \ \ \mathbf{c}_i=[c_{xi},c_{yi}]^\T \in \mathbb{R}^2
\end{align}
where $N$ is the number of points comprising the centerline, $\mathbf{c}_i$ for $i=1,\ldots,N$ are the pixel coordinates of $i$-th point represented in the camera frame.
In this article, the object is assumed to be tightly attached to the end-effector beforehand without sliding and falling off during the movement.
It means that the object grasping issue is not the considered topic in this work.
Slight movements $\mathbf{r}$ of the end-effector can cause changes in the shape $\bar{\mathbf{c}}$ of the object.
This complex nonlinear relationship is given as:
\begin{align}
\label{eq32}
\bar{\mathbf{c}} 
= \mathbf{f}_c(\mathbf{r}) 
= \mathbf{f}_c(\mathbf{f}_r(\mathbf{q}))
\end{align}

Note that the dimension $2N$ of the observed centerline $\bar{\mathbf{c}}$ is generally large, thus it is inefficient to directly use it in a shape controller as it contains redundant information.
In this work, we just give this generation concept, i.e., shape feature extraction.
This module aims to extract efficient features $\mathbf{s}$ from the original data space and then map them into the low-dimensional feature space.
The relation between the robot's angles and such shape descriptor is modeled as follows:
\begin{align}
\label{eq33}
\mathbf{s} 
= \mathbf{f}_s(\bar{\mathbf{c}}) 
= \mathbf{f}_s(\mathbf{f}_c(\mathbf{f}_r(\mathbf{q})))
\end{align}

Taking the derivative of \eqref{eq33} with respect to $\mathbf{q}$, we obtain the first-order dynamic model:
\begin{equation}
\label{eq52}
\dot{\mathbf{s}} = \frac{{\partial \mathbf{f}_s}}{{\partial {\mathbf{q}}}}{{\dot{\mathbf{q}}}}
= {\mathbf{J}_s}\left( \mathbf{r} \right)\dot{\mathbf{q}}
\end{equation}
where ${\mathbf{J}_s}\left( \mathbf{r} \right) \in {\mathbb{R}^{p \times 6}}$ is named after the deformation Jacobian matrix (DJM), which relates the velocity of the joint angles with the shape feature changes.
As the physical knowledge of the flexible objects is hard to obtain in the practical environment, thus the traditional analytical-based methods are not applicable for the calculation of DJM here.
To this end, the approximation methods are used to estimate DJM online for the real-time environment.
The quasi-static \eqref{eq52} holds when the materials properties of objects do not change significantly during manipulations, as $\mathbf{J}_s(\mathbf{r})$ represents the velocity mapping between the objects and the robot motions.

\begin{remark}
\label{remark1}
The nonlinear kinematic relationship between ${\mathbf{s}}$ and $\mathbf{q}$ can seen as the special form of the traditional standard robot kinematic Jacobian mapping, which captures the perspective geometry of the object's points in the boundary.
\end{remark}

\begin{remark}
\label{reamkr2}
Note that the deformations of the elastic objects (not considering rheological objects) are only related to its own potential energy, the contact force with the manipulator and without the manipulation sequence of the manipulator.
These conditions guarantee the establishment of \eqref{eq33}.
\end{remark}

\subsection{Input-Saturation Model}
The system is without the effect of the input saturation, i.e., do not consider the maximum velocity of the robot in the applications.
Although the above assumptions can simplify the design process of the system and reduce the algorithm complexity of the system, this affects the stability and deformation accuracy of the system.
The existence of DJM estimation error could lead to control failure.

Input saturation as the most critical non-smooth nonlinearity should be explicitly considered in the control design.
Thus we propose the sliding mode deformation control for the uncertain nonlinear system with estimation error case and unknown non-symmetric input saturation in this section.
The control input is $\mathbf{u}=[u_1,\ldots,u_q]^\T \in \mathbb{R}^q$.

For simplicity, we define that $\mathbf{u}=\dot{\mathbf{q}}$ in the following sections.
In the past articles, most of them adopted a hard-saturation manner, i.e., the influence of input saturation is not considered in theoretical analysis.
However, in the real environment, the speed of the robot is bounded.
Therefore, we need to consider the speed limit of the manipulator in the actual manipulation, to improve the stability of the system, which is very important for the manipulation of the deformable object.
The input saturation is defined as follows:
\begin{equation}
\label{eq21}
{u_i} = \left\{ {\begin{array}{*{20}{c}}
	{{u}_i^{\max }}&{{v_i} \ge u_i^{\max }}\\
	{{v_i}}&{u_i^{\min } < {v_i} < u_i^{\max }}\\
	{u_i^{\min }}&{{v_i} \le u_i^{\min }}
	\end{array}} \right.
	, for \ i=1,\ldots,6
\end{equation}
where $\mathbf{u} = [u_1,\ldots,u_6]^\T \in \mathbb{R}^6$ is the system input and the saturation output, and $\mathbf{v} = [v_1,\ldots,v_6]^\T \in \mathbb{R}^6$ is the actually designed control input.
${u}_i^{max}, {u}_i^{min}$ are the known joint angular velocity limits.

\subsection{Mathematical Properties}
Before furthering our control design, some useful properties are given here.
\begin{lemma}
\label{lemma1}
\cite{polycarpou1993robust}
The inequality $0 \le \left| x \right| - x\tanh \left({{x}/{\varepsilon}} \right) \le \delta \varepsilon$ holds for any $\varepsilon > 0$ and for any $x \in \mathbb{R}$, where $\delta=0.2785$ is a constant that satisfies $\delta=e^{-(\delta+1)}$.
\end{lemma}

\begin{definition}
\label{definition1}
\cite{ma2014adaptive}
Gauss error function $\erf(x)$ is a nonelementary function of sigmoid shape, which is defined as:
\begin{align}
    \label{eq2}
    \erf(x)=\frac{2}{\sqrt{\pi}} \int_{0}^{x} {e^{-t^2}dt}
\end{align}
where $\erf(x)$ is a real-valued and continuous differentiable function, it has no singularities (except that at infinity) and its Taylor expansion always converges.
\end{definition}

\begin{assumption}
\label{assumption1}
DJM is composed of estimated value $\hat{\mathbf{J}}_s(t)$ and approximation error $\tilde{\mathbf{J}}_s(\mathbf{r},t)$, i.e., $\hat{\mathbf{J}}_s(\mathbf{r}) = \hat{\mathbf{J}}_s(t) + \tilde{\mathbf{J}}_s(\mathbf{r},t)$.
\end{assumption}

\begin{assumption}
\label{assumption2}
The disturbance $\mathbf{d} = \hat{\mathbf{J}}_s \tilde{\mathbf{u}} + \tilde{\mathbf{J}}_s \mathbf{u}$ has the unknown positive constant limit, $\| \mathbf{d}\| \le \eta_1$.
\end{assumption}

\begin{assumption}
\label{assumption3}
The disturbance $\dot{\hat{\mathbf{J}}} \tilde{\mathbf{u}}$ has the unknown positive constant limit, $\| {\dot{\hat{\mathbf{J}}}_s \tilde{\mathbf{u}}} \| \le \eta_2$.
\end{assumption}

\begin{assumption}
\label{assumption4}
As the slow deformation speed is considered in this paper, thus it is assumed that $\dot{\tilde{\mathbf{u}}} = \mathbf{0}_p$.
\end{assumption}

\begin{figure}[ht]
\centering
\includegraphics[scale=0.11]{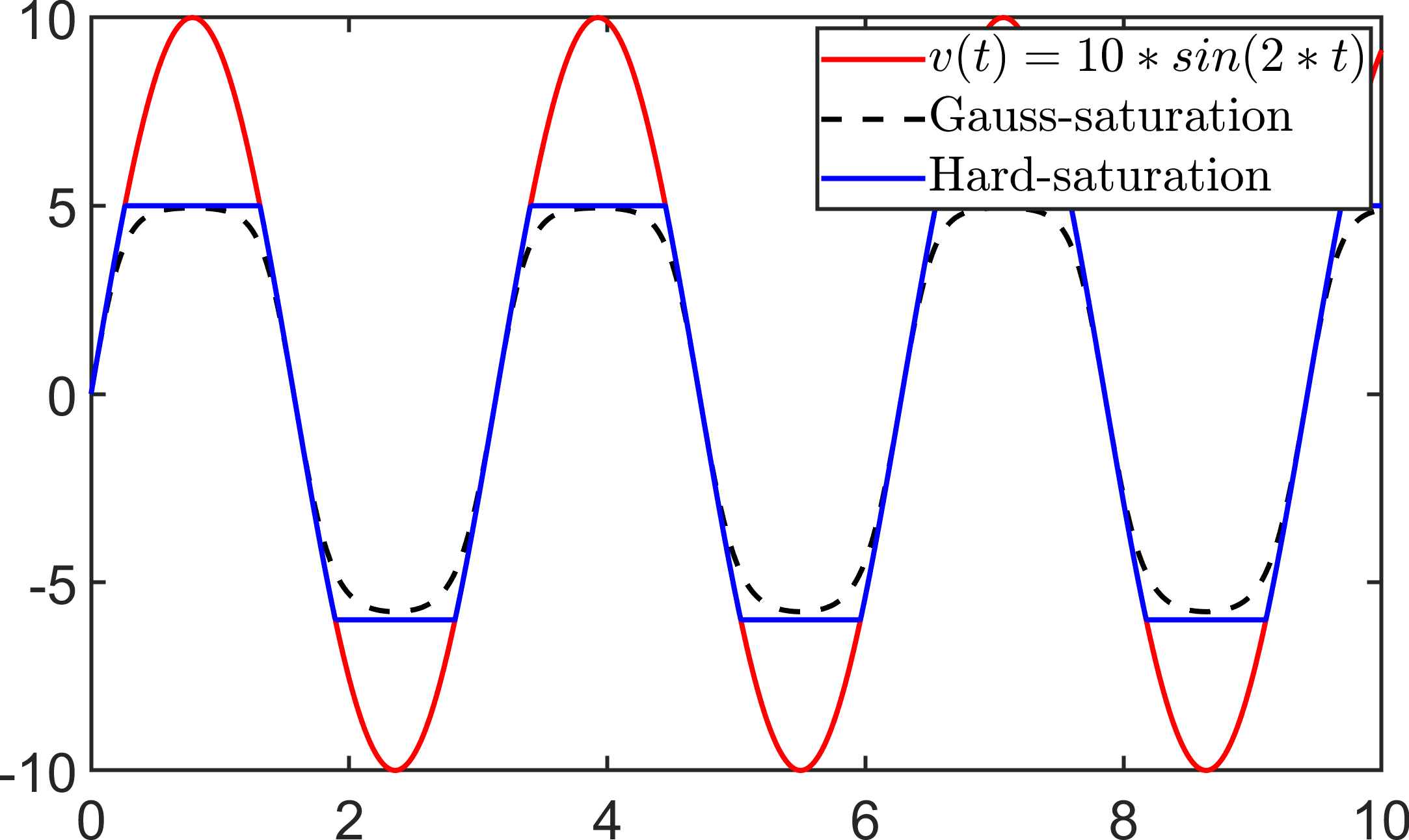}
\caption{Saturation functions.}
\label{fig1}
\end{figure}

\section{Controller Design}
Traditional input-saturation model \eqref{eq21} is a discontinuous function.
If this saturation model is simply adopted, it will affect the stability of the system, and it is easy to cause damage to the actuator in practical applications.
In this article, we utilize the model in \cite{ma2014adaptive} to describe the saturation nonlinearity.
Refer to Definition \ref{definition1}, the asymmetric input-saturation model can be transformed into the smooth format given as:
\begin{align}
\label{eq23}
u_i(v_i) 
&= {u}_{mi} \times \erf(\frac{\sqrt{\pi}}{2{u}_{mi}}{v}_i), \ \ for\ \ i=1,\ldots,6 , \\
{u}_{mi}
&=(u_i^{\max} + u_i^{\min})/2 + (u_i^{\max} - u_i^{\max})/2 * \sgn(v_i) \notag
\end{align}
where $\sgn(\cdot)$is the sign function.
Fig. \ref{fig1} shows the conceptual Gauss-saturation model in the following case:
\begin{equation}
v(t)=10\sin(2t), \ \ \
u^{\max}=5, \ \ \
u^{\min}=-6
\end{equation}
where $v(t)$ is the original control input without any saturation processing.
The saturation error function is constructed as follows:
\begin{align}
\label{eq24}
\tilde{\mathbf{u}} = \mathbf{u} - \mathbf{v} 
\end{align}
for $\mathbf{u}$ and $\mathbf{v}$ are the functions of time $t$.
Considering Assumption \ref{assumption1} and input-saturation model \eqref{eq23}, the differential equation \eqref{eq33}, the following anti-saturation model is obtained as follows:
\begin{equation}
\label{eq22}
\dot{\mathbf{s}}  
= \mathbf{J}_s\mathbf{u} 
= ( {\hat{\mathbf{J}}_s + \tilde{\mathbf{J}}_s} ) ( {\mathbf{v} + \tilde{\mathbf{u}}} )
= \hat{\mathbf{J}}_s \mathbf{v} + \mathbf{d}
\end{equation}
where $\mathbf{d} = \hat{\mathbf{J}}_s \tilde{\mathbf{u}} + \tilde{\mathbf{J}}_s \mathbf{u}$ is the total disturbance including approximation error $\tilde{\mathbf{J}}$ and saturation error $\tilde{\mathbf{u}}$.

Define the deformation error:
\begin{equation}
\label{eq3}
\begin{array}{*{20}{c}}
{{\mathbf{e}_1} = \mathbf{s} - {\mathbf{s}_d},} &
{{\mathbf{e}_2} = \dot{\mathbf{s}} - \hat{\mathbf{J}}_s \mathbf{u}}
\end{array}
\end{equation}
for $\mathbf{s}_d$ as the desired shape feature.
The derivatives of \eqref{eq3} with respect to time-variable $t$ is:
\begin{align}
\label{eq6}
{{\dot{\mathbf{e}} }_1} &= \hat{\mathbf{J}}_s \mathbf{v} + \mathbf{d} - {{\dot{\mathbf{s}} }_d} \notag \\
{{\dot{\mathbf{e}} }_2} &= \ddot{\mathbf{s}} 
- \dot{\hat{\mathbf{J}}}_s  \mathbf{v} 
- \hat{\mathbf{J}}_s \dot{\mathbf{v}} 
- \dot{\hat{\mathbf{J}}}_s \tilde{\mathbf{u}} 
- \hat{\mathbf{J}}_s \dot{\tilde{\mathbf{u}}}
\end{align}

Considering Assumption \ref{assumption4}, $\dot{\mathbf{e}}_2$ is transformed into:
\begin{equation}
\label{eq7}
{{\dot{\mathbf{e}} }_2} 
= \ddot{\mathbf{s}} - \dot{\hat{\mathbf{J}}}_s  \mathbf{v} - \hat{\mathbf{J}}_s \dot{\mathbf{v}} - \dot{\hat{\mathbf{J}}}_s \tilde{\mathbf{u}}
\end{equation}

The integral sliding surface is constructed as follows \cite{shen2014integral}:
\begin{align}
    {\boldsymbol{\sigma} _1} &= {\mathbf{e}_1} - {\mathbf{e}_1}\left( 0 \right) + \int_0^t {{\mathbf{e}_1}\left( \tau  \right)d\tau } \notag \\
    {\boldsymbol{\sigma} _2} &= {\mathbf{e}_2} - {\mathbf{e}_2}\left( 0 \right) + \int_0^t {{\mathbf{e}_2}\left( \tau  \right)d\tau }
\end{align}

Combing with \eqref{eq6} and \eqref{eq7}, the time derivative of $\boldsymbol{\sigma}_1$ is:
\begin{align}
\label{eq8}
\dot{\boldsymbol{\sigma}}_1 
&= {\hat{\mathbf{J}}_s \mathbf{v} + \mathbf{d} - {{\dot{\mathbf{s}}}_d} + \mathbf{e}_1} \notag \\
\dot{\boldsymbol{\sigma}}_2 
&= \ddot{\mathbf{s}} - \dot{\hat{\mathbf{J}}}_s  \mathbf{v} - \hat{\mathbf{J}}_s \dot{\mathbf{v}} - \dot{\hat{\mathbf{J}}}_s \tilde{\mathbf{u}} + {\mathbf{e} _2}
\end{align}
and design the velocity control input as follows:
\begin{align}
\label{eq9}
\mathbf{v}
&= {{\hat{\mathbf{J}}}_s^ + }\left( { - {\boldsymbol{\sigma} _1} + {{\dot{\mathbf{s}}}_d} - {{{\mathbf{e}}}_1} + {\boldsymbol{\Theta} _1}} \right) \notag \\
{\boldsymbol{\Theta} _1}
&= - \boldsymbol{\sigma} _1^{\T + }\tanh \left( {{{\left\| {{\boldsymbol{\sigma} _1}} \right\|}}/{{{\varepsilon _1}}}} \right){{\hat{\eta} }_1}\left\| {{\boldsymbol{\sigma} _1}} \right\|
\end{align}
where $\hat{\mathbf{J}}_s^+$ denotes the pseudo-inverse of the matrix $\hat{\mathbf{J}}_s$.
Since there is no power term and sign function, $\mathbf{v}$ is continuous without chattering.
And $\hat{\eta}_1$ is updated as:
\begin{align}
\label{eq35}
{{\dot{\hat \eta} }_1}
= \tanh \left( {{{\left\| {{\boldsymbol{\sigma} _1}} \right\|}}/{{{\varepsilon _1}}}} \right)\left\| {{\boldsymbol{\sigma} _1}} \right\| - {\gamma _1}{{\hat \eta}_1} 
\end{align}

To quantify the shape deformation error, we introduce the quadratic function $V_1(\boldsymbol{\sigma}_1) = \frac{1}{2} \boldsymbol{\sigma} _1^T \boldsymbol{\sigma} _1$ whose time-derivative satisfies:
\begin{align}
\label{eq10}
{{\dot V}_1}
&=\boldsymbol{\sigma}_1 ({\hat{\mathbf{J}}_s \mathbf{v} + \mathbf{d} - {{\dot{\mathbf{s}}}_d} + \mathbf{e}_1}) \notag \\
&= - {\left\| {{\boldsymbol{\sigma} _1}} \right\|^2} 
- \tanh ( {{{\left\| {{\boldsymbol{\sigma} _1}} \right\|}}/{{{\varepsilon _1}}}}){{\hat \eta}_1}\left\| {{\boldsymbol{\sigma} _1}} \right\| + \boldsymbol{\sigma} _1^\T \mathbf{d}
\end{align}

From Assumption \ref{assumption2} and considering norm induction formula, we can obtain the following relation:
\begin{equation}
\label{eq11}
\boldsymbol{\sigma} _1^\T \mathbf{d} \le \left\| {{\boldsymbol{\sigma} _1}} \right\|\left\| \mathbf{d} \right\| \le {\eta _1}\left\| {{\boldsymbol{\sigma} _1}} \right\|
\end{equation}

Substitution of \eqref{eq11} into \eqref{eq10} yields:
\begin{align}
\label{eq12}
{{\dot V}_1} \le
- {\left\| {{\boldsymbol{\sigma} _1}} \right\|^2} 
- \tanh ( {{{\left\| {{\boldsymbol{\sigma} _1}} \right\|}}/{{{\varepsilon _1}}}}){{\hat \eta}_1}\left\| {{\boldsymbol{\sigma} _1}} \right\| 
+ {\eta _1}\left\| {{\boldsymbol{\sigma} _1}} \right\|
\end{align}
% \begin{align}
% \label{eq12}
% {{\dot V}_1} \le 
% - {\left\| {{\sigma _1}} \right\|^2} 
% + {\eta _1}\delta {\varepsilon _1} 
% + {{\tilde{\eta} }_1}\tanh ( {\frac{{\left\| {{\sigma _1}} \right\|}}{{{\varepsilon _1}}}} )\left\| {{\sigma _1}} \right\| 
% \end{align}

Following the similar manner, the adaptive update rule of DJM is computed as:
\begin{align}
\label{eq13}
\dot{\hat{\mathbf{J}}}_s
&= ( {\ddot{\mathbf{s}}
- \hat{\mathbf{J}}_s \dot{\mathbf{v}}
+ {\boldsymbol{\sigma} _2} + {\mathbf{e}_2}
+ \boldsymbol{\sigma} _2^{T + }{\boldsymbol{\Theta} _2}} 
){\mathbf{v}^+} \notag \\
{\boldsymbol{\Theta} _2} 
&= \tanh \left( {{{\left\| {{\boldsymbol{\sigma} _2}} \right\|}}/{{{\varepsilon _2}}}} \right){{\hat{\eta} }_2}\left\| {{\boldsymbol{\sigma} _2}} \right\|
\end{align}
where the adaptive update rule of $\hat{\eta}_2$ is designed as follows:
\begin{align}
\label{eq14}
{{\dot{\hat \eta} }_2} 
= \tanh \left( {{{\left\| {{\sigma _2}} \right\|}}/{{{\varepsilon _2}}}} \right)\left\| {{\sigma _2}} \right\| - {\gamma _2}{{\hat \eta}_2}
\end{align}
for $\varepsilon_1, \varepsilon_2, \gamma_1, \gamma_2$ as positive constants.
$\boldsymbol{\Theta}_2$ is used to compensate for the effect of saturation error $\tilde{\mathbf{u}}$ on the estimated value of $\hat{\mathbf{J}}_s$.
The quadratic function $V_2(\boldsymbol{\sigma}_2) = \frac{1}{2} \boldsymbol{\sigma} _2^T \boldsymbol{\sigma} _2$ whose time-derivative satisfies:
\begin{align}
    \label{eq37}
    \dot{V}_2 
    &= \boldsymbol{\sigma}_2 (\ddot{\mathbf{s}} - \dot{\hat{\mathbf{J}}}_s  \mathbf{v} - \hat{\mathbf{J}}_s \dot{\mathbf{v}} - \dot{\hat{\mathbf{J}}}_s \tilde{\mathbf{u}} + {\mathbf{e} _2}) \notag \\
    &= - {\left\| {{\boldsymbol{\sigma} _2}} \right\|^2} - \tanh \left( {\left\| {{\boldsymbol{\sigma} _2}} \right\|/{\varepsilon _2}} \right){{\hat \eta }_2}\left\| {{\boldsymbol{\sigma} _2}} \right\| - \boldsymbol{\sigma} _2^\T{{\dot{\hat{\mathbf{J}} } }_s}\tilde{\mathbf{u}} 
\end{align}

From Assumption \ref{assumption3} and considering norm induction formula, we can obtain the following relation:
\begin{align}
    \label{eq38}
    \boldsymbol{\sigma} _2^\T{{\dot{\hat{\mathbf{J}}} }_s}\tilde{\mathbf{u}} \le {\eta _2}\left\| {{\boldsymbol{\sigma} _2}} \right\|
\end{align}

Substitution of \eqref{eq38} into \eqref{eq37} yields:
\begin{align}
    \label{eq39}
    \dot{V}_2 \le
    - {\left\| {{\boldsymbol{\sigma} _2}} \right\|^2} - \tanh \left( {\left\| {{\boldsymbol{\sigma} _2}} \right\|/{\varepsilon _2}} \right){{\hat \eta }_2}\left\| {{\boldsymbol{\sigma} _2}} \right\| + {\eta _2}\left\| {{\boldsymbol{\sigma} _2}} \right\|
\end{align}

\begin{proposition}
Consider the closed-loop shape servoing system \eqref{eq22} considered with Assumption \ref{assumption1} - \ref{assumption4}, the input-saturation issue \eqref{eq13}, the velocity controller \eqref{eq9}, the DJM estimation law \eqref{eq13}, with adaptation laws \eqref{eq35} \eqref{eq14}.
For a given desired feature vector $\mathbf{s}_d$, there exists an appropriate set of control parameters that ensure that:
\begin{enumerate}
\item 
All signals in the close-loop system remain uniformly ultimately bounded (UUB);

\item 
The deformation error $\mathbf{e}_1$ asymptotically converges to a compact set around zero.
\end{enumerate}
\end{proposition}

\begin{proof}
Consider the energy-like function
\begin{equation}
\label{eq15}
V(\boldsymbol{\sigma}_1, \boldsymbol{\sigma}_2, \tilde{\eta}_1, \tilde{\eta}_2) 
= V_1(\boldsymbol{\sigma}_1) + V_2(\boldsymbol{\sigma}_2)
+ \frac{\tilde{\eta} _1^2}{2}
+ \frac{\tilde{\eta} _2^2}{2}
\end{equation}
for 
$\tilde{\eta}_1=\eta_1-\hat{\eta}_1,\tilde{\eta}_2 = \eta_2 - \hat{\eta}_2$ are the approximation errors of the constant $\eta_1$ and $\eta_2$, respectively.
Thus, the time-derivative of \eqref{eq15} is computed as:
\begin{align}
    \label{eq40}
    \dot V = {{\dot V}_1}\left( {{\boldsymbol{\sigma} _1}} \right) + {{\dot V}_2}\left( {{\boldsymbol{\sigma} _2}} \right) 
    - {{\tilde \eta }_1}{{\dot{\hat{\eta}}}_1}
    - {{\tilde \eta }_2}{{\dot{\hat{\eta}}}_2}
\end{align}

Invoking \eqref{eq12} \eqref{eq39}, we can show that the time derivative of $V_2$ satisfies:
\begin{align}
    \label{eq41}
    \dot V 
    &\le - {\left\| {{\sigma _1}} \right\|^2} - {\left\| {{\sigma _2}} \right\|^2} - {{\tilde \eta }_1}{{\dot{\hat{\eta} }  }_1} - {{\tilde \eta }_2}{{\dot{\hat{\eta} }  }_2} \notag \\
    &- \tanh \left( {\left\| {{\sigma _1}} \right\|/{\varepsilon _1}} \right){{\hat \eta }_1}\left\| {{\sigma _1}} \right\| + {\eta _1}\left\| {{\sigma _1}} \right\| \notag \\
    &- \tanh \left( {\left\| {{\sigma _2}} \right\|/{\varepsilon _2}} \right){{\hat \eta }_2}\left\| {{\sigma _2}} \right\| + {\eta _2}\left\| {{\sigma _2}} \right\|
\end{align}

Thus, we can obtain:
\begin{align}
    \label{eq42}
    \dot V 
    &\le  - {\left\| {{\sigma _1}} \right\|^2} - {\left\| {{\sigma _2}} \right\|^2} \notag \\
    &- {{\tilde \eta }_1}( {{{\dot{\hat{\eta} }  }_1} - \tanh \left( {\left\| {{\sigma _1}} \right\|/{\varepsilon _1}} \right)\left\| {{\sigma _1}} \right\|} ) \notag \\
    &- {{\tilde \eta }_2}( {{{\dot{\hat{\eta} }  }_2} - \tanh \left( {\left\| {{\sigma _2}} \right\|/{\varepsilon _2}} \right)\left\| {{\sigma _1}} \right\|} ) \notag \\
    &+ {\eta _1}\left( {\left\| {{\sigma _1}} \right\| - \left\| {{\sigma _1}} \right\|\tanh \left( {\left\| {{\sigma _1}} \right\|/{\varepsilon _1}} \right)} \right) \notag \\
    &+ {\eta _2}\left( {\left\| {{\sigma _2}} \right\| - \left\| {{\sigma _2}} \right\|\tanh \left( {\left\| {{\sigma _2}} \right\|/{\varepsilon _2}} \right)} \right)
\end{align}

Refer to Lemma \ref{lemma1}, it yields
\begin{align}
\label{eq17}
{{\dot V}} 
&\le - {\left\| {{\sigma _1}} \right\|^2} - {\left\| {{\sigma _2}} \right\|^2} 
 + {\eta _1}\delta {\varepsilon _1} + {\eta _2}\delta {\varepsilon _2} \notag \\
&+ {{\tilde{\eta} }_1}(  \left\| {{\sigma _1}} \right\| {\tanh \left( {{{\left\| {{\sigma _1}} \right\|}}/{{{\varepsilon _1}}}} \right) 
	- {{\dot{\hat{\eta} } }_1}} ) \notag \\
&+ {{\tilde{\eta} }_2}( {\left\| {{\sigma _2}} \right\|\tanh \left( {{{\left\| {{\sigma _2}} \right\|}}/{{{\varepsilon _2}}}} \right) 
	- {{\dot{\hat{\eta} } }_2}} ) 
\end{align}

By the well-known Young's inequality, it yields
\begin{equation}
\label{eq16}
\begin{array}{*{20}{c}}
{{{\tilde{\eta} }_1}{{\hat{\eta} }_1} \le \frac{1}{2}\eta_1^2 - \frac{1}{2}\tilde \eta_1^2},&{{{\tilde \eta}_2}{{\hat \eta}_2} \le \frac{1}{2}\eta_2^2 - \frac{1}{2}\tilde \eta_2^2}
\end{array}
\end{equation}

Considering the adaptive rules \eqref{eq35} \eqref{eq14}, it yields
\begin{align}
\label{eq18}
{{\dot V}} 
&\le  - {\left\| {{\sigma _1}} \right\|^2} - {\left\| {{\sigma _2}} \right\|^2} - \frac{{{\gamma _1}}}{2}\tilde \eta_1^2 - \frac{{{\gamma _2}}}{2}\tilde \eta_2^2 \notag \\
&+ \frac{{{\gamma _1}}}{2}\eta_1^2 + \frac{{{\gamma _2}}}{2}\eta_2^2 + {\eta_1}\delta {\varepsilon _1} + {\eta_2}\delta {\varepsilon _2} \notag \\
&\le  - aV + b 
\end{align}
where $a=\min(2,\gamma_1,\gamma_2)$, $b=\frac{{{\gamma _1}}}{2}\eta_1^2 + \frac{{{\gamma _2}}}{2}\eta_2^2 + \delta({\eta_1} {\varepsilon _1} + {\eta_2} {\varepsilon _2})$.
By selecting the appropriate $\gamma_1$ and $\gamma_2$ that ensure that $a>0$, the deformation error $\mathbf{e}_1$ asymptotically converge to a compact set around zeros, and ensures that the estimation error $\tilde{\mathbf{J}}_s$ remains bounded. 
\end{proof}

\appendices
\ifCLASSOPTIONcaptionsoff
  \newpage
\fi

\bibliography{biblio.bib}
\bibliographystyle{IEEEtran}
\end{document}